\documentclass[11pt]{article}

\usepackage{fullpage}

\usepackage[utf8]{inputenc} 
\usepackage[T1]{fontenc}    
\usepackage{hyperref}       
\usepackage{url}            
\usepackage{booktabs}       
\usepackage{amsfonts}       
\usepackage{nicefrac}       
\usepackage{microtype}      

\usepackage{amsmath, amsthm}
\usepackage{wrapfig}
\usepackage{thmtools}
\usepackage{thm-restate}

\usepackage{enumitem}
\usepackage[]{algorithm}
\usepackage{algorithmic}

\title{Online Linear Optimization with Many Hints}

\author{
  Aditya Bhaskara \\
  University of Utah \\
  Salt Lake City, UT \\
  \texttt{bhaskaraaditya@gmail.com} 
   \\
   \\
   Ravi Kumar \\
   Google Research \\
   Mountain View, CA \\
   \texttt{ravi.k53@gmail.com} 
   
   \and
   
   Ashok Cutkosky \\
   Boston University \\
   Boston, MA \\
   \texttt{ashok@cutkosky.com} 
   \\
   \\
   Manish Purohit \\
   Google Research \\
   Mountain View, CA \\
   \texttt{mpurohit@google.com}
}

\usepackage{color}

\newcommand{\mymath}[1]{\centerline{$\displaystyle{#1}$}}

\newcommand{\bR}{\mathbb{R}}

\newcommand{\norm}[1]{\left\Vert #1 \right\Vert}

\newcommand{\R}{\mathbb{R}}
\newcommand{\E}{\mathbb{E}}

\newcommand{\iprod}[1]{\langle #1 \rangle}

\newcommand{\cB}{\mathcal{B}}
\newcommand{\cA}{\mathcal{A}}
\newcommand{\cC}{\mathcal{C}}
\newcommand{\rand}{\mathrm{rand}}
\newcommand{\cS}{\mathcal{S}}

\newcommand{\su}[1]{^{(#1)}}

\newcommand{\xbar}{\overline{x}}

\newcommand{\argmin}{\mathop{\text{argmin}}}

\newcommand{\inner}[2]{\langle #1, #2 \rangle}

\newcommand{\Rcc}[1]{\cS(#1)}
\newcommand{\Rcci}[1]{\cS_{#1}}
\newcommand{\Rc}[3]{\cS([#1, #2], #3)}
\newcommand{\Rci}[4]{\cS_{#1}([#2, #3], #4)}

\newtheorem{theorem}{Theorem}

\newtheorem{defn}[theorem]{Definition}
\newtheorem{corr}[theorem]{Corollary}

\newcommand{\singlehint}{\textsc{$1$-Hint}}

\newcommand{\manyhint}{\textsc{$K$-Hints}}

\newcommand{\regret}{\mathcal{R}}
\newcommand{\reals}{\mathbb{R}}

\newcommand{\MW}{\text{MW}}

\newcommand{\mnote}[1]{\textcolor{red}{[MP: #1]}}

\newif\ifarxiv
\arxivtrue

\date{}
\begin{document}

\maketitle

\begin{abstract}
We study an online linear optimization (OLO) problem in which the learner is provided access to $K$ ``hint'' vectors in each round prior to making a decision. In this setting, we devise an algorithm that obtains logarithmic regret whenever there exists a convex combination of the $K$ hints that has positive correlation with the cost vectors. This significantly extends prior work that considered only the case $K=1$. To accomplish this, we develop a way to combine  many arbitrary OLO algorithms to obtain regret only a logarithmically worse factor than the minimum regret of the original algorithms in hindsight; this result is of independent interest.
\end{abstract}

\section{Introduction}
\label{sec:intro}

In this paper we consider a variant of the classic online linear optimization (OLO) problem~\cite{zinkevich2003online}.  In OLO, at each time step, an algorithm must play a point $x_t$ in some convex set $X \subseteq \R^d$, and then it is presented with a cost vector $c_t$ and incurs loss $\iprod{c_t, x_t}$. This process repeats for $T$ time steps.  The algorithm's performance is measured via the \emph{regret} relative to some comparison point $u\in X$, defined as $\sum_{t=1}^T \langle c_t, x_t-u\rangle$.  

This problem is of fundamental interest in a variety of fields.  OLO algorithms are directly applicable for solving the learning with expert advice problem as well as online convex optimization~\cite{CLBook}.  Further, in machine learning, one frequently encounters stochastic convex optimization problems, which may be solved via online convex optimization through the online-to-batch conversion~\cite{cesa2002generalization}.  Many of the popular optimization algorithms used in machine learning practice today (e.g.,~\cite{duchi2011adaptive, DBLP:journals/corr/KingmaB14}) can be analyzed within the OLO framework.  For more details and further applications, we refer the interested reader to the excellent texts~\cite{CLBook,HazanBook,shalev2012online}.

OLO is well-understood from an algorithmic viewpoint.  For the vanilla version of the problem, algorithms with regret $O(\sqrt{T})$ are known~\cite{zinkevich2003online, KV05} and this bound is tight~\cite{CLBook}.  An interesting line of research has been to identify situations and conditions where the regret can be substantially smaller than $\sqrt{T}$.   Towards this, Dekel et al.~\cite{DBLP:conf/nips/DekelFHJ17} proposed the study of OLO augmented with hints; their work was motivated by an earlier work of Hazan and Megiddo~\cite{DBLP:conf/colt/HazanM07}.  In their setup, the algorithm has access to a hint at each time step before it responds and this hint is guaranteed to be more than $\alpha$-correlated with the cost vector.  They obtained an algorithm with a regret of $O(d/\alpha \cdot \log T)$, where $d$ is the dimension of the space.   Very recently, Bhaskara et al.~\cite{bhaskara2020online} generalized their results to the case when the hints can be arbitrary, i.e., not necessarily weakly positively correlated at each time step.  They obtain an algorithm with a dimension-free regret bound that is roughly  $O(\sqrt{B}/\alpha \cdot \log T)$, where $B$ is the number of (bad) time steps when the hints are less than $\alpha$-correlated with the cost vector.

While this line of work gives a promising way to go beyond $\sqrt{T}$ regret, in many situations, it is not clear how to obtain a hint sequence that correlates well with the cost vector in most time steps. Prior work on optimism \cite{rakhlin2013online,hazan2010extracting, steinhardt2014adaptivity} has suggested using costs from earlier time steps, costs from earlier batches, or even from other learning algorithms. This suggests that it is often possible to obtain {\em multiple} sources that provide hint sequences, and we may hope that an appropriate combination of them correlates well with the cost vector in most time steps.  

In this work, we focus on this natural setting in which multiple (arbitrary) hints are available to the algorithm at each time step. If some aggregate of the hints is helpful, we would like to perform as well as if we knew this aggregate a priori. As we discuss in Section \ref{sec:smooth-hinge}, this is difficult because the benefit of aggregating multiple hints is a nonlinear function of the benefits of the individual hints. Even if all the hints are individually bad, an algorithm may be able to gain significantly from using some convex combination of the hints.

\paragraph{Our results.}

Let $K$ be the number of hints available at each time step. 
We obtain an online learning algorithm for the constrained case, where the responses of the algorithm must be inside the unit ball.  Our algorithm obtains a regret of roughly $O(\sqrt{B/\alpha \cdot \log T} + (\log T + \sqrt{(\log T)(\log K)})/\alpha)$, where $B$ is the number of time steps when the \emph{best} convex combination of the hints is less than $\alpha$-correlated with the cost vector.  We refer to Theorem~\ref{thm:mainconstrained} for the formal guarantee. We also obtain lower bounds showing the dependence of the regret on both $K$ and $\alpha$ is essentially tight (Section~\ref{sec:lb-3}).

Our algorithm is designed in two stages.  In the first stage, we assume that the optimal threshold $\alpha$ is known.  We build an algorithm based on carefully defining a \emph{smoothed} hinge loss function that captures the performance over the entire simplex of hints and then using Mirror Descent on the losses.  The second stage eliminates the assumption on knowing $\alpha$ by developing a new \emph{combiner} algorithm.  This is a general randomized procedure that combines a collection of online learning algorithms and achieves regret only logarithmically worse than the minimum regret of the original algorithms.  
(This combiner is of independent interest and we show a few applications outside our main theme.)  

For the unconstrained setting (defined formally below), we develop an algorithm that achieves a (relative) regret of roughly $O(\log T \cdot (\sqrt{B/\alpha} + \sqrt{\log K}/\alpha))$, where $B$ is once again defined as before.  Our algorithm thus competes with the best convex combination of the hints.    

\section{Preliminaries}
\label{sec:prelim}

Let $[T] = \{1, \ldots, T\}$.  In the classical online learning setting, at each time $t \in [T]$, an algorithm $\cA$ responds with a vector $x_t \in \bR^d$. \emph{After} the response, a cost vector $c_t \in \bR^d$ is revealed and the algorithm incurs a cost of $\inner{c_t}{x_t}$. We assume that $\|c_t\| \leq 1, \ \forall t\leq T$, where $\|\cdot\|$ always indicates the $\ell_2$-norm unless specified otherwise. The \emph{regret} of the algorithm $\cA$ for a vector $u \in \bR^d$ is 

\mymath{
\regret_\cA(u, \vec{c}) = 
\regret_\cA(u, \vec{c}, T) = \sum_{t = 1}^T \inner{c_t}{x_t - u}.
}

A \emph{hint} is a vector $h \in \reals^d$, $\|h\| \leq 1$ and $\vec{h} = (h_1, h_2, \ldots)$ is a sequence of hints.  We consider the case when there are \emph{multiple} hints available to the algorithm $\cA$.  In each round $t$, the algorithm $\cA$ gets $K$ hints $h_t\su{1}, \ldots, h_t\su{K}$ \emph{before} it responds with $x_t$.  While some of the hint sequences might be good and others might be misleading, our goal is to design an algorithm that does nearly as well as if we were just given the best sequence of hints.  Let $H = \{\vec{h}^{(1)}, \ldots, \vec{h}^{(K)}\}$ denote the set of hint sequences.  The regret definition is the same as always and is denoted $\regret_\cA(u, \vec{c} \mid H)$. 

Let $\Delta_K\subset \R^K$ denote the simplex.  Given a sequence $\vec{w} = (w_1, w_2, \ldots)$ of vectors in $\Delta_K$, we write $H(\vec{w})$ to indicate the sequence of hints with $t$th hint $\sum_{i=1}^K w_t^{(i)} \cdot h_t^{(i)}$, 
where $w_t^{(i)}$ indicates the $i$th component of $w_t$.  If $\vec{w}$ is a constant sequence $(w, w, \ldots)$, then we write $H(w)$ instead of $H(\vec{w})$.  


Let $\alpha > 0$ be a fixed \emph{threshold}. For a fixed hint sequence $\vec{h}$, we define $B^{\vec{h}}_{\alpha}$ to be the set of all time steps where the hint $h_t$ is \emph{bad}, i.e., less than $\alpha$ correlated with the cost $c_t$. Formally, we have
$$B^{\vec{h}}_{\alpha} = \left\{ t \in [T] : \inner{c_t}{h_t} < \alpha \cdot \|c_t\|^2 \right\}.$$  


\newcommand{\galpha}[1]{G^{\vec{#1}}_{\alpha}}
\newcommand{\balpha}[1]{B^{\vec{#1}}_{\alpha}}



We consider two settings to measure the worst-case regret of an algorithm.  In the \emph{constrained} setting, we are given some set $\cB$ and the worst-case regret of $\cA$ is defined as $\regret_\cA(\cB, \vec{c} \mid H) = \sup_{u\in \cB}\regret_\cA(u, \vec{c} \mid H)$;  in this paper we take $\cB=\{x\in \bR^d:\ \|x\|\le 1\}$, the unit ball.  In the \emph{unconstrained} setting, the regret of $\cA$ is measured over $u \in \reals^d$ and we denote it by
$\regret_{\cA}(u,\vec{c} \mid H)$, which we will bound uniformly by another function of $u$.


\subsection{Single hint case}
Now we recall and mildly improve the results of \cite{bhaskara2020online} for the case that there is a \emph{single} hint at every time step (i.e., $K=1$).  We will consider the case of fixed and \emph{known} $\alpha$; note that the algorithm of \cite{bhaskara2020online} is agnostic to $\alpha$, but we show that by committing to $\alpha$ we can improve the regret bound. We will remove this dependence on a known $\alpha$ later in Section~\ref{sec:combiner}. The modification to both the algorithm and the analysis is not hard, and so we defer the proof to Appendix~\ref{sec:betteralphasinglehint}.  

\begin{restatable}{theorem}{thmbetteralphasinglehint}\label{thm:betteralphasinglehint}
For any $0 < \alpha < 1$, there exists an algorithm $\singlehint_\alpha$ that runs in $O(d)$ time per update, takes a single hint sequence $\vec{h}$, and guarantees regret:
\begin{align*}
    \regret_{\singlehint_\alpha}(\cB, \vec{c} \mid \{\vec{h}\})&\le \frac{1}{2}+4\left(\sqrt{ \sum_{t\in B^{\vec{h}}_{\alpha}}\|c_t\|^2} + \frac{\log T}{\alpha}+2\sqrt{\frac{(\log T)\sum_{t=1}^T \max(0, -\langle c_t, h_t\rangle)}{\alpha}}\right)\\
    &\le O\left(\sqrt{\frac{(\log T)|B^{\vec{h}}_\alpha|}{\alpha}} + \frac{\log T}{\alpha}\right).
\end{align*}
\end{restatable}
In contrast, the bound in~\cite{bhaskara2020online} had the factor $(\log T)/\alpha$ instead of $\sqrt{(\log T)/\alpha}$ (in the first term).

\section{Constrained setting: Known $\alpha$}
\label{sec:constrained}

Recall that in the constrained setting, the algorithm must always respond with $x_t \in \cB$, the unit ball. Our main result is a version of Theorem~\ref{thm:betteralphasinglehint} for $K > 1$, and it will extend the previous works of~\cite{DBLP:conf/nips/DekelFHJ17, bhaskara2020online}. The high-level approach is quite natural: we design a \emph{meta-learner} that maintains a loss for each hint sequence at each time, and at time $t$, uses the losses to decide on an appropriate convex combination $w_t$ of the hints $\{h^i_t\}_{i=1}^K$. We then run an instance of the single hint algorithm, $\singlehint_\alpha$, using this combination as the provided hint. 

There are two main challenges with this approach. First, the regret bound of $\singlehint_\alpha$ depends on the quantity $B_{\alpha}^{H(\vec{w})}$, which depends on the convex combination $w_t$ used at each step $t$, and it is not clear how to relate it to the corresponding terms for the individual hint sequences. Second, the regret bound assumes a knowledge of $\alpha$, while our final goal is to compete with the best possible (unknown) $\alpha$.  We deal with the second challenge in Section~\ref{sec:combiner}
by designing a general combination algorithm.  In this section we address the first challenge; all the algorithms in this section assume a fixed and known value of $\alpha$. Any omitted proofs may be found in Appendix \ref{app:constrained}.

\subsection{Multiplicative weights on hint sequences}
 
We first show a result weaker than the main result of this section (Theorem~\ref{thm:mainconstrained}). The algorithm is conceptually simpler, and it demonstrates what one obtains by using a simple multiplicative weight update (MWU) rule to learn the best hint sequence among the $K$ sequences, and then use Theorem \ref{thm:betteralphasinglehint} with the learned hint sequence. Since the single-hint regret bound (Theorem~\ref{thm:betteralphasinglehint}) depends on just the number of time steps when the hint has a poor correlation with the cost vector, using an MWU algorithm using binary losses suffices. In particular, if $\vec{h}^{\MW}$ denotes the hint sequence obtained from the multiplicative weights algorithm, we can show that $|B^{\vec{h}^{\MW}}_{\alpha}| \leq O(\min_{i \in K} |B^{\vec{h}^{(i)}}_{\alpha}|)$. 

\begin{restatable}{theorem}{mwuthm}\label{thm:mwuthm}
Let $\alpha > 0$ be given. There exists a randomized algorithm $\cA_\MW$ for OLO with $K$ hint sequences that has a regret bound of
\begin{align*}
    \E[\regret_{\cA_{\MW}} (\cB, \vec{c} \mid H)] \le O\left( \inf_{i \in K}  \sqrt{\frac{(\log T) (|B^{\vec{h}^{(i)}}_\alpha| + \log K)}{\alpha}}+\frac{\log T}{\alpha}\right).
\end{align*}
\end{restatable}

Note that this is usually weaker than Theorem~\ref{thm:mainconstrained} because it competes only with the best individual hint sequence, and not necessarily the best convex combination of hints. It can only be a better bound if $K\gg T$ so that $\log K = \omega(\log T)$.

\ifarxiv
\begin{proof}
At each time step $t$, our goal is to pick a single hint $h_t \in \{h_t^{(1)}, \ldots, h_t^{(K)}\}$. We instantiate this problem as an instance of the standard prediction with $K$ experts problem with binary losses defined as follows.
\begin{align*}
    \ell_{t, i} = 
        \begin{cases}
        0 \quad\quad \text{if } |\iprod{c_t, h_t^{(i)}}| \geq \alpha \norm{c_t}, \\
        1 \quad\quad \text{otherwise.}
        \end{cases}
\end{align*}

Let $\vec{h}^{(i^*)}$ denote the hint sequence with minimum loss in hindsight, i.e., $i^* = \argmin_{i \in K} \sum_t \ell_{t, i}$. We note that by definition of the losses $\ell$, we have $\sum_t \ell_{t, i^*} = |B^{\vec{h}^{(i^*)}}_{\alpha}|$.
Let $\vec{h}^{\text{MW}} = (h_1^{(i_1)}, h_2^{(i_2)}, \ldots)$ be the sequence of hints obtained by running the classical Multiplicative Weights algorithm with a decay factor of $\eta = \frac{1}{2}$. 
Then by standard analysis (e.g., Theorem 2.1 of Arora et al.~\cite{arora2012multiplicative}), we have the following.
\begin{align}
    \E[\sum_{t} (\ell_{t, i_t} - \ell_{t, i^*})] \leq 2\log K + \frac{1}{2} \sum_{t} \left(\ell_{t, i^*}\right).
    \intertext{Substituting $|B^{\vec{h}^{(i^*)}}_{\alpha}| = \sum_{t} \ell_{t, i^*}$ and rearranging,}
    \E[|B^{\vec{h}^{\text{MW}}}_{\alpha}|] = \E[\sum_{t} \ell_{t, i_t}] \leq \frac{3}{2} |B^{\vec{h}^{(i^*)}}_{\alpha}| + 2 \log K. \label{eq:MWbound}
\end{align}

We then run an instance of the single hint algorithm, $\singlehint_\alpha$, with the hint sequence $\vec{h}^{\text{MW}}$. Applying Theorem \ref{thm:betteralphasinglehint} yields the following.
\begin{align*}
    \E[\regret_{A_{\MW}}(\cB, \vec{c} \mid H)] &\le  O\left(\E\left[\sqrt{\frac{(\log T)|B^{\vec{h}^{\MW}}_\alpha|}{\alpha}}\right] + \frac{\log T}{\alpha}\right)\\
    &\le O\left(\sqrt{\frac{(\log T)\E\left[|B^{\vec{h}^{\MW}}_\alpha|\right]}{\alpha}} + \frac{\log T}{\alpha}\right)\\
    &\le O\left(\sqrt{\frac{(\log T)(|B^{\vec{h}^{(i^*)}}_{\alpha}| +  \log K)}{\alpha}} + \frac{\log T}{\alpha}\right),
\end{align*}
where the first inequality follows from Jensen's inequality and the second one follows from~\eqref{eq:MWbound}.
\end{proof}
\fi

\subsection{Smoothed Hinge Loss}\label{sec:smooth-hinge}

\ifarxiv
\else
\begin{wrapfigure}{R}{0.5\textwidth}
\begin{minipage}{0.5\textwidth}
\centering
\vspace*{-5mm}
\fi
\begin{algorithm}[H]
   \caption{$\manyhint_\alpha$}
   \label{alg:manyehintalpha}
   \begin{algorithmic}
   \REQUIRE Parameter $\alpha$
    \STATE Define $\psi(w) = (\log K)+\sum_{i=1}^K w^{(i)} (\log w^{(i)})$
    \STATE Initialize $\singlehint_{\alpha/2}$
    \STATE Initialize $w_1 \leftarrow (1/K,\dots,1/K)\in \Delta_K$
    \FOR{$t=1,\dots,T$}
    	\STATE Get hints $h^{(1)}_t,\dots h^{(K)}_t$
    	\STATE Send $h_t \leftarrow \sum_{i=1}^K w_t^{(i)} h^{(i)}_t$ to $\singlehint_{\alpha/2}$
    	\STATE Get $x_t$ from $\singlehint_{\alpha/2}$.
    	\STATE Respond $x_t$, receive cost $c_t$
    	\STATE Send $c_t$ to $\singlehint_{\alpha/2}$
    	\STATE $\ell_t(w) \leftarrow \ell\left(\iprod{c_t, \sum_{i=1}^K w^{(i)} h^{(i)}_t}, \alpha \|c_t\|^2\right)$
    	\STATE $g_t \leftarrow \nabla \ell_t(w_t)$
    	\STATE $w_{t+1} \leftarrow \argmin_{w\in \Delta_K}\langle g_{1:t}, w\rangle + \sqrt{\frac{(\log K)+\sum_{\tau=1}^t\|g_\tau\|^2_\infty}{\log K}}\psi(w)$
	\ENDFOR
\end{algorithmic}
\end{algorithm}
\ifarxiv
\else
\end{minipage}
\end{wrapfigure}
\fi

The multiplicative weights approach allows us to obtain regret guarantees that depend on the number of bad hints in the best of the $K$ hint sequences. But, what we would really like is for the regret bound to scale with the number of bad hints in the best \emph{convex combination} of the hint sequences. This can be a significant gain: consider the setting in which $K=2$ and $\alpha = \frac{1}{4}$, and on even iterations $t$ we have $\langle c_t, h^{(1)}_t\rangle = -1/4$ while on odd iterations $\langle c_t, h^{(1)}_t\rangle = 1$. Suppose $h^{(2)}_t$ is the same, but has high correlation on even iterations and negative correlation on odd iterations. Then both $h^{(1)}_t$ and $h^{(2)}_t$ have $T/2$ ``bad hints'', but the convex combination $(\frac{h^{(1)}_t}{2}+\frac{h^{(2)}_t}{2})$ has \emph{no} bad hints!
This highlights the fundamental problem with the multiplicative weights approach: linear combinations of hints might result in much better performance than the corresponding linear combination of the respective performances of the hints.
%
%
%

We will address this issue by considering a specially crafted loss function that more accurately captures performance over the entire simplex of hints. Intuitively, we would like to design a loss function such that for any $w \in \Delta_K$, the loss $\ell_t(w)$ is low if and only if $h_t(w) = \sum_{i=1}^K h_t^{(i)}w^{(i)}$ has the desired correlation with $\|c_t\|^2$. Once we have the appropriate loss function, we can then use an online learning algorithm on the losses $\ell_t$ to obtain the desired convex combination of hints at each time step.

Formally, the following smoothed version of the hinge loss is adequate for our purposes.
\begin{align}
    \ell(a,b) = \left\{\begin{array}{lr}
    0& a>b\\
    \frac{1}{b}(b-a)^2& a\in[0,b]\\
    b - 2a& a<0\end{array}\right. \label{eq:smoothhingeloss}
\end{align}

For any $w \in \Delta_K$, we define the loss function as $\ell_t(w) = \ell(\iprod{c_t, h_t(w)}, \alpha \|c_t\|^2)$ where $h_t(w) = \sum_{i=1}^K w^{(i)} h_t^{(i)}$ and $\ell(\cdot)$ is as defined in~\ref{eq:smoothhingeloss}.
We first present several important properties of this loss function in the following proposition.

\begin{restatable}{proposition}{smoothprop}\label{thm:smoothprop}
Let $\alpha$ be a constant and 
for any $t \in [T]$, let 
$\ell_t(w) = \ell(\iprod{c_t, h_t(w)}, \alpha \|c_t\|^2)$. Then, 

\begin{enumerate}[nosep]
\item[(a).] $\ell_t$ is convex and non-negative.
\item[(b).] If $h_t(w)$ is $\alpha$-good (i.e., $\langle c_t, h_t(w)\rangle \ge \alpha \|c_t\|^2$), then $\ell_t(w)=0$ and $0\in \partial \ell_t(w)$.
\item[(c).] If $h_t(w)$ is not ($\alpha/2$)-good (i.e., $\langle c_t, h_t(w)\rangle < \alpha\|c_t\|^2/2$), then $\ell_t(w)\ge \alpha\|c_t\|^2/4$. 
\item[(d).] $\ell_t$ is 2-Lipschitz with respect to the $\ell_1$-norm.
\item[(e).] $\|\nabla \ell_t(w)\|_\infty^2\le \frac{4}{\alpha}\ell_t(w)$ for all $w\in \Delta_K$.
\item[(f).] $\ell_t(w) \le \alpha \|c_t\|^2 +2\max(0,-\langle c_t, h_t(w)\rangle)$.
\end{enumerate}
\end{restatable}

\ifarxiv
\begin{proof}
Properties (a)--(c) are immediate from the definition of $\ell(\cdot, \cdot)$.

For the next properties, define $f: \reals \rightarrow \reals$ by $f(x) = \ell(x, \alpha \|c_t\|^2)$. 
By manually computing derivatives of $f$ we can see that $f$ is $2$-Lipschitz and $1$-smooth.
Further since $|\langle c_t, h^{(i)}_t\rangle|\le 1$ for all $i$, we have that $g(w) = \iprod{c_t, h_t(w)}$ is 1-Lipschitz with respect to the $\ell_1$-norm. Therefore $\ell_t$ must be 2-Lipschitz with respect to the $\ell_1$-norm, proving (d).

By inspecting the derivatives of $f$, we see that $f'(x)^2\le \frac{4}{\alpha\|c_t\|^2}f(x)$. Further, we have $\nabla \ell_t(w)^{(i)} = \langle c_t, h^{(i)}_t\rangle f'(\langle c_t, h_t(w)\rangle)$. Therefore $\|\nabla \ell_t(w)\|_\infty\le \|c_t\|f'(\langle c_t, h_t(w)\rangle)$, from which (e) follows.
For (f), we observe that $f(x)\le \alpha \|c_t\|^2 + 2\max(0,-x)$.
\end{proof}

\fi

We are now ready to present our final algorithm $\manyhint_\alpha$. At each timestep $t$, we first choose a $w_t \in \Delta_K$ via FTRL on the losses $\ell_t$. We then supply the learned hint $h_t(w_t) = \sum_{i=1}^K w_t^{(i)} h_t^{(i)}$ to an instance of the single hint algorithm. For technical reasons, we use the single hint algorithm $\singlehint_{\alpha/2}$ where the desired correlation the cost vector is set to $\alpha / 2$ instead of $\alpha$. Algorithm \ref{alg:manyehintalpha} presents the pseudocode of the entire algorithm.
The performance of the FTRL subroutine can be bounded via classical results in FTRL (see \cite{mcmahan2017survey}) used in concert with the smoothness of the losses $\ell_t$, following \cite{srebro2010smoothness}. The final result is the following Proposition \ref{thm:selfbounding}, which we prove in Appendix~\ref{app:constrained}.

\begin{restatable}{proposition}{selfbounding}\label{thm:selfbounding}
Let $w_t \in \Delta_K$ be chosen via FTRL on the losses $\ell_t$.  Then, for any $w_\star \in \Delta_K$, we have
\begin{align*}
    \sum_{t=1}^T \ell_t(w_t)\le \frac{22\log K}{\alpha} +  2\sum_{t=1}^T \ell_t(w_\star). 
\end{align*}
\end{restatable}

With this proposition, we can prove the main result of this section:
\begin{restatable}{theorem}{mainconstrained}\label{thm:mainconstrained}
Let $\alpha >0$ be given. Then $\manyhint_\alpha$ on OLO with $K$ hint sequences guarantees:
\begin{align*}
    \regret_{\manyhint_\alpha}(\cB, \vec{c} \mid H)&\le O\left( \inf_{w\in\Delta_K} \sqrt{ (\log T)\sum_{t\in B^{H(w)}_\alpha} \|c_t\|^2}+ \sqrt{\frac{(\log T)\sum_{t=1}^T\max(0,-\langle c_t, h_t(w)\rangle)}{\alpha}}\right.\\
    &\kern17em\left.+\frac{(\log T)+\sqrt{(\log T)(\log K)}}{\alpha} \right) \\
    &\le O\left( \inf_{w \in \Delta_K} \sqrt{\frac{(\log T) |B^{H(w)}_\alpha|}{\alpha}}+\frac{(\log T) + \sqrt{(\log T)(\log K)}}{\alpha}\right).
\end{align*}
In the above, $h_t(w)=\sum_{i=1}^K w^{(i)}h_t^{(i)}$ is the $t$th hint of the sequence $H(w)$ for $w\in \Delta_K$.
\end{restatable}
\begin{proof}
Let $w_\star$ be an arbitrary element of $\Delta_K$. By Proposition~\ref{thm:smoothprop}(f), we have $\ell_t(w_\star)\le \alpha \|c_t\|^2 + 2\max(0,-\langle c_t, h_t(w_\star)\rangle)$ for all $t$, and $\ell_t(w_\star)=0$ if $\langle c_t, h_t(w_\star)\rangle \ge \alpha \|c_t\|^2$. Therefore,
\begin{align}
    \sum_{t=1}^T \ell_t(w_\star)\le \sum_{t\in B^{H(w_\star)}_{\alpha}} \left(\alpha \|c_t\|^2+ 2\max(0,-\langle c_t, h_t(w_\star)\rangle)\right) = Q, \label{eq:Qdef}
\end{align}
where we have defined the variable $Q=\sum_{t\in B^{H(w_\star)}_{\alpha}} \alpha \|c_t\|^2+ 2\max(0,-\langle c_t, h_t(w_\star)\rangle)$.

Further, by definition of the smoothed hinge loss, we have $\ell_t(w_t) \geq \max(0, -\iprod{c_t, h_t(w_t)})$ for all $t \in [T]$. Therefore, by Proposition~\ref{thm:selfbounding} and~\eqref{eq:Qdef}, we have
\begin{align}
    \sum_{t=1}^T \max\left(0, -\langle c_t, h_t(w_t)\rangle\right) \le  \sum_{t=1}^T \ell_t(w_t) \leq 2Q + \frac{22\log K}{\alpha}. \label{eq:term1}
\end{align}

Also, since the loss function is always non-negative, we have 
\begin{align}
\sum_{t=1}^T \ell_t(w_t) \geq \sum_{t \in B_{\alpha/2}^{H(\vec{w})}} \ell_t(w_t) \geq \sum_{t \in B_{\alpha/2}^{H(\vec{w})}} \frac{\alpha \|c_t\|^2}{4}. \nonumber
\intertext{where the second inequality uses Proposition~\ref{thm:smoothprop}(c). Once again, using Proposition~\ref{thm:selfbounding} and~\eqref{eq:Qdef}, we have}
\sum_{t \in B_{\alpha/2}^{H(\vec{w})}} \|c_t\|^2 \leq \frac{8Q}{\alpha} + \frac{88\log K}{\alpha^2}. \label{eq:term2}
\end{align}



Finally, recall that we have sent the hint sequence $H(\vec{w}) = (h_1(w_1), \ldots, h_T(w_T))$ to the algorithm $\singlehint_{\alpha/2}$. Thus by Theorem \ref{thm:betteralphasinglehint}, we have:
%
\begin{align}
    \regret_{\manyhint_{\alpha}} (\cB, \vec{c}&\mid H) \le \frac{1}{2}+4\left(\sqrt{  \sum_{t\in B^{H(\vec{w})}_{\alpha/2}} \|c_t\|^2} + \frac{\log T}{\alpha}+\sqrt{\frac{(2 \log T)\sum_{t=1}^T \max(0, -\langle c_t, h_t(w_t)\rangle )}{\alpha}}\right), \nonumber\\
    \intertext{substituting~\eqref{eq:term1} and~\eqref{eq:term2},}
    &\le \frac{1}{2}+4\left(\sqrt{ \frac{8Q}{\alpha}+\frac{88\log K}{\alpha^2}} + \frac{\log T}{\alpha}+\sqrt{\frac{2 (\log T)\left(2Q + \frac{22\log K}{\alpha}\right)}{\alpha}}\right). \label{eq:regret-khints}
\end{align}
The final result now follows from the definition of $Q$ and simple calculations.
\end{proof}

\paragraph{Non-negatively correlated hints.} 
Recall that in the case of $K=1$,~\cite{DBLP:conf/nips/DekelFHJ17} obtains a regret of $O((\log T)/\alpha)$ in the case where {\em all} the hints are $\alpha$-correlated with $c_t$. A weaker assumption is to have $\iprod{h_t, c_t} \ge 0$ at all steps, with the $\alpha$-correlation property holding at all but $B_\alpha$ time steps. In this case,~\cite{bhaskara2020online} showed that the regret must be at least $\Omega(\sqrt{B_\alpha})$, and also gave an algorithm that achieves a regret of $O\left(\sqrt{B_\alpha} + \frac{\log T}{\alpha} \right)$. Using Theorem~\ref{thm:mainconstrained}, we obtain this bound for general $K$.



\begin{corr}\label{thm:dekel-extension}
Consider OLO with $K$ hint sequences where for every $t$ and every hint $h_t^{(i)}$, we have the property that $\iprod{h_t^{(i)}, c_t} \ge 0$. Further, suppose that for some $\alpha > 0$, there exists an (unknown) convex combination $w$ such that for the hint sequence $H(w)$, the number of hints that do not satisfy $\iprod{h_t(w), c_t} \ge \alpha \norm{c_t}^2$ is at most $B_\alpha$. Then there exists an algorithm that achieves a regret at most
\[ O\left( \sqrt{B_\alpha} + \frac{ \log T+\sqrt{\log K}}{\alpha} \right).  \]
\end{corr}
This follows from the proof of Theorem~\ref{thm:mainconstrained}. Specifically, before substituting to obtain~\eqref{eq:regret-khints}, observe that under the non-negative correlation assumption, $\max (0, \iprod{c_t, h_t(w_t)}) = 0$ for all $t$, and thus we only have the first two terms of~\eqref{eq:regret-khints}. This gives the desired bound.


\subsection{Lower bounds}\label{sec:lb-3}

In this section we provide some lower bounds, focusing on the dependence on $K$ and $\alpha$.  Our primary technique is to specify hint sequences and costs such that, even given the hint, the cost is $\alpha$-correlated with some combination of hints, but otherwise is a random variable with mean 0 and variance 1. The high variance in the costs guarantees nearly $\sqrt{T}$ regret, which we express in terms of $\alpha$ and $K$ to achieve our bounds. 
We begin with a lower bound showing that the dependence on $\sqrt{(\log K)/\alpha}$ holds even in one dimension.
\begin{restatable}{theorem}{logklower}\label{thm:logklower}
For any $\alpha$ and $T\ge \frac{1}{\alpha} \log \frac{1}{\alpha}$, there exists a sequence $\vec{c}$ of costs and a set $H$ of hint sequences, $|H| = K$ for some $K$, such that:
(i) there is a convex combination of the $K$ hints that always has correlation $\alpha$ with the costs and (ii) the regret of any online algorithm is at least $\sqrt{\frac{\log K}{2\alpha}}$.
\end{restatable}

\ifarxiv
\begin{proof}
Consider a one-dimensional problem with $K=\frac{T2^B}{B}$ hint sequences for $B=\alpha T$. Suppose $T\ge \frac{\log(1/\alpha)}{\alpha}$, so that $2^B\ge \frac{T}{B}$ and $\log K\le 2B=2T\alpha$. We group the hint sequences into $\frac{T}{B}$ groups each of size $2^B$.
We now specify the hint sequence in the $i$th such group for some arbitrary $i$.  All hints in the $i$th group are 0 for all $t\notin [(i-1)B,iB-1]$ and for $t \in [iB, (i+1)B)$, the hints take on the $2^B$ possible sequences of $\pm1$. Then it is clear that for \emph{any} sequence of $\pm1$ costs, there is a convex combination of hints that places weight $B/T$ on exactly one hint sequence in each of the $T/B$ groups such that the linear combination always has correlation $\alpha=B/T$ with the cost. 

Let the costs be random $\pm 1$, so that the expected regret is $\sqrt{T}$. Then we conclude by observing $\sqrt{\log K}/\sqrt{2\alpha}\le \sqrt{2\alpha T}/\sqrt{2\alpha}=\sqrt{T}$.
\end{proof}
\fi

Next, we show that some dependence on $1/\alpha$ is unavoidable:
\begin{restatable}{theorem}{constrainedalpha}\label{thm:constrainedalpha}
In the two-dimensional constrained setting, there is a sequence $\vec{h}$ and $\vec{c}$ of hints and costs such that: (i)  $\forall t$, $\iprod{h_t, c_t} \ge \alpha$, and
(ii) the regret of any online  algorithm is at least $\Omega(1/\alpha)$.  
\end{restatable}

\ifarxiv
\begin{proof}
Let $e_0$ and $e_1$ be orthogonal unit vectors, and let $h_t = e_0$ for all $t$. Suppose that $c_t = \alpha e_0 \pm \sqrt{1-\alpha^2} e_1$ for all $t$, where the sign is chosen uniformly at random.  Note that any online algorithm has expected reward at most $\alpha T$ (since it cannot gain anything in the $e_1$ direction, so it is best to place all the mass along $e_0$).

On the other hand, we have

\mymath{
\mathbb{E} \left[ \norm{\sum_{t = 1}^T c_t}^2 \right] = \alpha^2 T^2 + T (1-\alpha^2), 
}
and thus the optimal vector in hindsight achieves a reward $\sqrt{\alpha^2 T^2 + T(1-\alpha^2)}$. Thus the regret is

\mymath{
\frac{T (1-\alpha^2)}{\alpha T + \sqrt{\alpha^2 T^2 + T(1-\alpha^2)}} \ge \frac{T (1-\alpha^2)}{2\alpha T + \sqrt{T(1-\alpha^2)}} \geq \frac{1}{\alpha},
}
for sufficiently large $T$.
\end{proof}

\fi


\section{Combining learners}
\label{sec:combiner}

\ifarxiv
\begin{figure}
\else
\begin{wrapfigure}{R}{0.5\textwidth}
\begin{minipage}{0.5\textwidth}
\fi
\centering
\vspace*{-5mm}
\begin{algorithm}[H]
\caption{Deterministic combiner $\cC_{\det}$.}\label{alg:combiner}
\begin{algorithmic}
   \STATE{\bfseries Input: } Online algorithms $\cA_1,\dots,\cA_K$
   \STATE Reset $\cA_1$
   \STATE Set $i \leftarrow 1$, $\gamma \leftarrow 1$, $r \leftarrow 0$, $\tau \leftarrow 1$, $r^{i,\gamma}_0 \leftarrow 0$
   \FOR{$t=1,\dots, T$}
   \STATE Get $y_\tau$ from $\cA_i$ and respond $x_t \leftarrow y_\tau$
   \STATE Get cost $c_t$, define $g_\tau \leftarrow c_t$
   \STATE Send $g_\tau$ to $\cA_i$ as $\tau$th cost
   \STATE Set $r^{i,\gamma}_\tau \leftarrow \sup_{u\in \cB}\sum_{\tau'=1}^\tau \langle g_{\tau'}, y_{\tau'}- u\rangle$
   \IF{$r^{i,\gamma}_\tau>\gamma$}
   \IF{$i = K$}
   \STATE Set $\gamma \leftarrow 2\gamma$
   \ENDIF
   \STATE Set $i \leftarrow (i$ mod $K) + 1$
   \STATE Set $\tau \leftarrow 1$
   \STATE Set $r^{i,\gamma}_0 \leftarrow 0$
   \STATE Reset $\cA_i$
   \ENDIF 
   \STATE Set $\tau \leftarrow \tau + 1$
   \ENDFOR
\end{algorithmic}
\end{algorithm}
\ifarxiv
\end{figure}
\else
\vspace*{-10mm}
\end{minipage}
\end{wrapfigure}
\fi

In Section~\ref{sec:constrained}, we presented an algorithm for online learning with multiple hints.  However, the algorithm required knowing $\alpha$, the desired correlation between a hint $h$ and the cost vector $c_t$.  In this section, we eliminate this assumption. To do this, we design a generic way to combine incomparable-in-foresight regret guarantees obtained by different algorithms and essentially get the best regret among them in hindsight. With this combiner, handling unknown $\alpha$ is easy: consider $\manyhint_{\alpha}$ for different values of $\alpha$ and apply the combiner to get the best among them.

The results in this section apply in the constrained setting and to both the hints and the classical no-hints case (see \cite{cutkosky2019combining} for analogous results that apply only in the unconstrained setting). These combiner algorithms themselves are of independent interest and lead to other applications in the constrained online learning setting that we elaborate in \ifarxiv Section \fi~\ref{sec:disc}.

For technical reasons, we need the following definition of a ``monotone regret bound''. Essentially all regret bounds known for online linear optimization satisfy this definition.
\begin{defn}[Monotone regret bound]
An online learning algorithm $\cA$ is associated with a \emph{monotone regret bound} $\Rc{a}{b}{\vec{c}}$, if $\Rcc{\cdot, \cdot}$ is such that when $\cA$ is run on only the costs $c_a,\dots,c_b$, producing outputs $x_a,\dots,x_b$, we have the guarantee:
\[
\sup_{u\in \cB}\sum_{t=a}^{b} \langle c_t, x_t - u\rangle \le  \Rc{a}{b}{\vec{c}},
\]
and further it satisfies $\Rc{a'}{b'}{\vec{c}} \le \Rc{a}{b}{\vec{c}}$ for all sequences $\vec{c}$ whenever $[a',b']\subseteq [a,b]$.
\end{defn}
Note that if an algorithm $\cA$ has a monotone regret bound $\Rcc{\cdot, \cdot}$, then $\regret_{\cA}(\cB, \vec{c}) = \Rc{1}{T}{\vec{c}}$.

\subsection{Deterministic combiner}

We first design a simple deterministic algorithm $\cC_{\det}$ that combines $K$ online learning algorithms with monotone regret bounds and obtains a regret that is at most $K$ times the regret suffered by the best algorithm on any given cost sequence.  The combiner starts with an initial guess of the regret $\gamma$ and guesses that the first algorithm is the best, playing its predictions. It keeps trusting the current choice of the best algorithm until the regret it incurs exceeds the current guess $\gamma$; once that happens, it chooses the next algorithm.  Once all the algorithms have been tried, it doubles the guess $\gamma$ and starts over.  
Notice that this does not require knowledge of the bounds $\Rcci{i}$; these can be replaced with the ``true'' regret bounds, rather than simply the best bound that present analysis is capable of delivering.

\begin{restatable}{theorem}{Kcombiner}\label{thm:Kcombiner}
Suppose $\cA_1,\dots,\cA_K$ are deterministic OLO algorithms that are associated with monotone regret bounds $\Rcci{1},\dots,\Rcci{K}$. Suppose $\forall t$, $\sup_{x,y\in \cB} \langle c_t, x-y\rangle \le 1$. Then,
we have:

\mymath{
    \regret_{\cC_{\det}}(\cB, \vec{c})
    \le K\left(4 + 
    4\min_i \regret_{\cA_i} (\cB, \vec{c}) \right).
}
\end{restatable}

\begin{proof}[Proof sketch] We give a brief sketch here and defer the formal proof to Appendix~\ref{app:combiner}. We can divide the operation of Algorithm~\ref{alg:combiner} into phases in which $\gamma$ is constant. In each phase, Algorithm~\ref{alg:combiner} incurs a regret of at most $\gamma + 1$ from each of the $K$ algorithms for a total regret of at most $K(\gamma + 1)$. Let $P$ denote the total number of phases and let $j = \argmin_i  \Rci{i}{1}{T}{\vec{c}}$ be the algorithm with the least total regret. In the $(P-1)$th phase, algorithm $A_j$ must have incurred a regret of at least $2^{P-2}$ (otherwise we would not have the $P$th phase). Since we assume that $\Rcci{j}$ is a monotone regret bound, it follows that $\min_i  \Rci{i}{1}{T}{\vec{c}} \geq 2^{P-2}$ and hence $P \leq \max(1, 2 + \log_2(\min_i  \Rci{i}{1}{T}{\vec{c}}))$.
Since $\gamma = 2^{p-1}$ in phase $p$, we can bound the total regret incurred by Algorithm~\ref{alg:combiner} as
\begin{align*}
    \sup_{u\in \cB} \sum_{t=1}^T\langle c_t, x_t - u\rangle 
    &\le \enspace \sum_{p=1}^P K(2^{p-1} + 1)
    \enspace \le \enspace K(P + 2^{P})\le K2^{P+1}\\
    &\le K\left(4 + 4\min_i  \Rci{i}{1}{T}{\vec{c}}\right).\qedhere
\end{align*}
\end{proof}


\subsection{Randomized combiner}\label{sec:combiner-random}

The deterministic combiner $\cC_{\det}$, while achieving the best regret among $\cA_1, \ldots, \cA_K$, incurs a factor $K$.  We now show that using randomization, this factor can be made $O(\log K)$ in expectation. 

Intuitively, $\cC_{\det}$ incurs the factor $K$ since it might be unlucky and have to cycle through all the $K$ algorithms even after it correctly guesses $\gamma$.  We can avoid this worst-case behavior by selecting the base algorithm uniformly at random, rather than in a deterministic order. We formally describe this randomized combiner $\cC_{\rand}$ in Algorithm~\ref{alg:logkcombiner} in Appendix \ref{app:combiner}.  Informally, in each phase with constant $\gamma$, at each time step, $\cC_{\rand}$ simulates all the $K$ algorithms and maintains a candidate set $C$ of algorithms that have incurred a regret of at most $\gamma$. Once the current algorithm incurs a regret of $\geq \gamma$, $\cC_{\rand}$ selects the next algorithm to be one from the set $C$ uniformly at random. Suppose the algorithms in $C$ are ranked by the first time they incur a regret bound of $\gamma$. Since an algorithm $\cA_i$ is chosen uniformly at random, in expectation, by the time $\cA_i$ incurs a regret of $\gamma$, half of the algorithms in $C$ have already incurred at least $\gamma$ regret and thus the size of $C$ halves at each step. Thus, we can argue that we only cycle through $O(\log K)$ base algorithms in each phase. We defer the formal proof of the following theorem to Appendix~\ref{app:combiner}.





\begin{restatable}{theorem}{logkcombiner}\label{thm:logkcombiner}
Suppose $\cA_1,\dots,\cA_K$ are deterministic OLO algorithms with monotone regret bounds $\Rcci{1},\dots,\Rcci{K}$. Suppose for all $t$, $\sup_{x,y\in \cB} \langle c_t, x-y\rangle \le 1$. Then for any fixed sequence $\vec{c}$ of costs (i.e., an oblivious adversary), Algorithm \ref{alg:logkcombiner} guarantees:

\mymath{
    \E\left[ \regret_{\cC_{\rand}}(\cB, \vec{c})\right] 
    \le \log_2(K+1) \cdot \left( 4+
    4\min_i \regret_{\cA_i}(\cB, \vec{c}) \right).
}
Further, if $\vec{c}$ is allowed to depend on the algorithm's randomness (i.e., an adaptive adversary), then

\mymath{
    \regret_{\cC_{\rand}}(\cB, \vec{c})
    \le K\left( 4+
    4\min_i \regret_{\cA_i}(\cB, \vec{c}) \right).
}
\end{restatable}

\subsection{Constrained setting: Unknown $\alpha$}
\label{sec:unknownalpha}

For any fixed $\alpha > 0$, Theorem~\ref{thm:mainconstrained} yields a monotone regret bound.  For $1 \leq i \leq \log T$, let $\cA_i$ denote the instantiation of Algorithm~\ref{alg:manyehintalpha} with $\alpha_i = 2^{-i}$. By Theorem~\ref{thm:mainconstrained}, each algorithm $\cA_i$ is associated with a monotone regret bound $\Rcci{i}(\cdot, \cdot)$ such that

\mymath{
\regret_{\cA_i}(\cB, \vec{c}) = 
\Rci{i}{1}{T}{\vec{c}} = O\left( \inf_{w \in \Delta_K}
\sqrt{\frac{(\log T) |B^{H(w)}_{\alpha_i}|}{\alpha_i}}+\frac{(\log T) + \sqrt{(\log T)(\log K)}}{\alpha_i}
\right).
}

Further since $|B_{\alpha_{i+1}}^{H(w)}| \leq |B_{\alpha_{i}}^{H(w)}|$, we have $\Rcci{i+1}(\cdot, \vec{c}) \leq 2 \Rcci{i}(\cdot, \vec{c})$. Applying Theorem \ref{thm:logkcombiner} on these $\log T$ algorithms thus yields the following result.

\begin{restatable}{theorem}{unknownalpha}\label{thm:unknownalpha}
Given a set $H = \{\vec{h}^1, \ldots, \vec{h}^K\}$ of hint sequences, there exists a randomized algorithm $\cA$ such that for any fixed sequence of cost vectors $\vec{c}$, the expected regret $\E[\regret_{\cA}(\cB, \vec{c}~\mid~H)] $ is at most:

\mymath{
     O\left(\inf_{\alpha} \inf_{w \in \Delta_K}
     \left\{
     (\log\log T) \cdot \left(\sqrt{\frac{(\log T) |B^{H(w)}_\alpha|}{\alpha}}+\frac{(\log T) + \sqrt{(\log T)(\log K)}}{\alpha}
     \right\}
     \right)\right).
}
\end{restatable}

\ifarxiv
\section{Other applications of the combiner}
\label{sec:disc}

In this section we discuss a couple of direct applications of our combiner algorithms to other settings.

\subsection{Adapting to different norms}
\label{sec:normadaptive}

For any $\ell_p$-norm, $p\in(1,2]$, there is an algorithm that guarantees regret $\sup_{u\in \cB} \frac{\|u\|_p}{\sqrt{p-1}} \sqrt{\sum_{t=1}^T \|c_t\|_q^2}$ where $q$ is such that $\frac{1}{p} + \frac{1}{q} = 1$ (such bounds can be obtained by e.g., the adaptive FTRL analysis described in~\cite{mcmahan2017survey}, or see~\cite{shalev2012online} for a non-adaptive version). However, it is not clear which $p$-norm yields the best regret guarantee until we have seen all the costs.  Fortunately, these are monotone regret bounds, so by making a discrete grid of $O(\log d)$ $p$-norms in a $d$-dimensional space we can obtain the best of all these bounds in hindsight up to an additional factor of $\log d$ in the regret. Specifically:
\begin{restatable}{theorem}{manyp}\label{thm:manyp}
Let $K=\lfloor (\log d)/2\rfloor$, let $q_0=2$ and $\frac{1}{q_i}=\frac{1}{q_{i-1}} - \frac{1}{\log d}$ for $i\le K$. Define $p_i$ by $\frac{1}{q_i}+\frac{1}{p_i}=1$. For each $i\in [K]$, let $\cA_i$ be an online learning algorithm that guarantees regret $\sup_{u\in \cB} \frac{\|u\|_{p_i}}{\sqrt{p_i-1}} \sqrt{\sum_{t=1}^T \|c_t\|_{q_i}^2}$. Then combining these algorithms using Algorithm~\ref{alg:combiner} yields a worst-case regret bound of:

\mymath{
    \E[\regret_{\cA}(\cB, \vec{c})] 
    \enspace \le \enspace 
    O\left( (\log \log d) \cdot \inf_{p}\sup_{u\in \cB} \frac{\|u\|_p}{\sqrt{p-1}} \sqrt{\sum_{t=1}^T \|c_t\|_q^2}\right).
}
\end{restatable}

\subsection{Simultaneous Adagrad and dimension-free bounds}
\label{sec:dimensionfree-adagrad}
The adaptive online gradient descent algorithm of~\cite{hazan2008adaptive} obtains the regret bound $D_2\sqrt{\sum_{t=1}^T \|c_t\|_2^2}$, where $D_2$ is the $\ell_2$-diameter of $\cB$. In contrast, the Adagrad algorithm obtains the bound $D_\infty \sum_{i=1}^d \sqrt{\sum_{t=1}^T c_{t,i}^2}$ where $D_\infty$ is the $\ell_\infty$-diameter of $\cB$ and $c_{t,i}$ is the $i$th component of $c_t$~\cite{duchi2011adaptive}. Adagrad's bound can be extremely good when the $c_t$ are sparse, but might be much worse than the adaptive online gradient descent bound otherwise. However, both bounds are clearly monotone, so by applying our combiner construction, we have:
\begin{restatable}{theorem}{dimfreeadagrad}\label{thm:dimfreeadagrad}
There is an algorithm $\cA$ such that for any sequence of vectors $\vec{c}$, the regret is at most:

\mymath{
    \E[\regret_{\cA}(\cB, \vec{c})] 
    \enspace \le \enspace 
    O\left(\min\left\{ D_2\sqrt{\sum_{t=1}^T \|c_t\|_2^2},
    \enspace 
    D_\infty \sum_{i=1}^d \sqrt{\sum_{t=1}^T c_{t,i}^2}\right\} \right).
}
\end{restatable}


\fi

\section{Unconstrained setting}
\label{sec:uncons}
In this section, we develop an algorithm that leverages multiple hints in the unconstrained setting. Recall that in this setting, the output $x_t$ and comparison point $u$ are allowed to range over all of $\R^d$. Thus we cannot hope to bound regret by a uniform constant for all $u$. Instead, we bound the regret as a function of $\|u\|$. This setting has seen increased interest~\cite{orabona2014simultaneous,orabona2016coin,foster2018online,cutkosky2017online,cutkosky2018black}, and recently the notion of hints has also been studied~\cite{cutkosky2019combining, bhaskara2020online}. Here, we consider multiple hints in the unconstrained setting. Unlike the constrained case, this algorithm does not need to know $\alpha$ and hence does not need the combiner.  The algorithm again competes with the best convex combination of the hints. 

Following \cite{bhaskara2020online, cutkosky2019combining}, our algorithm initializes $K+1$ unconstrained online learners. The first online learner ignores the hints and attempts to output $x_t$ to minimize the regret. Each of the following $K$ online learners is restricted to output real numbers $y^{(i)}_t$ for $i=1,\dots,K$ rather than points in $\R^d$. The final output of our algorithm is then given by $\hat x_t = x_t +\sum_{i=1}^K y^{(i)}_t h^{(i)}_t$. Intuitively, the $i$th one-dimensional algorithm is attempting to learn how ``useful'' the $i$th hint sequence is. Upon receiving the cost $c_t$,  we provide the $i$th one-dimensional algorithm with the cost $\langle c_t, h^{(i)}_t\rangle$. Note that we are leaning heavily on the lack of constraints in this construction. 

\begin{restatable}{theorem}{thmunconstrained}\label{thm:unconstrained}
There is an algorithm $\cA$ for the unconstrained setting that achieves regret
\begin{align*}
    \regret_\cA(u, \vec{c} \mid H) = O\left(\inf_{w\in \Delta_K}\left\{ \|u\| (\log T) \left( \frac{\sqrt{\log K}}{\alpha} + \sqrt{\frac{B_{\alpha}^{H(w)}}{\alpha}} \right) \right\} \right).
\end{align*}
\end{restatable}

\ifarxiv

\begin{proof}
Algorithm $\cA$ instantiates one $d$-dimensional \emph{parameter-free} OLO algorithm $\cA'$ that outputs $x_t$, gets costs $c_t$, and guarantees regret for some user specified $\epsilon$:
\begin{align*}
    \sum_{t=1}^T \langle c_t, x_t -u\rangle \le \epsilon + O\left(\|u\|\log(T)+\|u\|\sqrt{\sum_{t=1}^T \|c_t\|^2\log \frac{T}{\epsilon}}\right).
\end{align*}
Where the $O$ hides absolute constants. Such algorithms are described in several recent works~\cite{cutkosky2018black,cutkosky2019matrix, van2019user, kempka2019adaptive, mhammedi2020lipschitz}. Also, algorithm $\cA$ instantiates $K$ one-dimensional learning algorithms, $\cA_i$ for the hint sequence $\vec{h^{(i)}}$.  At time $t$, the $i$th such learner outputs $y^{(i)}_t$, gets cost $-\langle c_t, h^{(i)}_t\rangle$ and guarantees regret:
\begin{align*}
    \sum_{t=1}^T \langle c_t, h^{(i)}_t\rangle(y^{(i)} - y^{(i)}_t)
    &\le \frac{\epsilon}{K} + O\left(|y^{(i)}|\log(T)+|y^{(i)}|\sqrt{\sum_{t=1}^T \langle c_t, h^{(i)}_t\rangle^2 \log \frac{KT}{\epsilon}} \right)\\
    & \le \enspace \frac{\epsilon}{K} + O\left(|y^{(i)}|\log(T)+|y^{(i)}|\sqrt{\sum_{t=1}^T \|c_t\|^2\log \frac{KT}{\epsilon}}\right).
\end{align*}
These one-dimensional learners may simply be instances of the $d$-dimensional learner restricted to one dimension.
The algorithm $\cA$ responds with the predictions $\hat x_t = x_t - \sum_{i=1}^K y^{(i)}_t h^{(i)}_t$ and set $\epsilon=1$.  The regret is:
\begin{align*}
    \sum_{t=1}^T \langle c_t, \hat x_t - u\rangle &=\sum_{t=1}^T \langle c_t, x_t -u\rangle -\sum_{i=1}^K\sum_{t=1}^T\langle c_t, h_t^{(i)}\rangle y_t^{(i)}\\
    &=\inf_{y^{(1)}, \ldots, y^{(K)}\in \R} \left\{ \sum_{t=1}^T \langle c_t, x_t -u\rangle+\sum_{i=1}^K\sum_{t=1}^T \langle c_t, h^i_t\rangle(y^{(i)} - y^{(i)}_t) - \sum_{t=1}^T \left\langle c_t, \sum_{i=1}^Ky^{(i)}h^{(i)}_t \right\rangle \right\} \\
    &  \le O\left(\inf_{y^{(1)}, \ldots, y^{(K)}\in \R} \left\{ 1 + \|u\|\sqrt{\sum_{t=1}^T \|c_t\|^2\log T} + \sum_{i=1}^K\left( \frac{1}{K} + |y^{(i)}|\sqrt{\sum_{t=1}^T \|c_t\|^2\log(KT)}\right)\right.\right.\\
    &\kern15em\left.\left.+\|u\|\log(T)+\sum_{i=1}^K |y^{(i)}|\log(T)-\sum_{t=1}^T \left\langle c_t, \sum_{i=1}^Ky^{(i)}h^{(i)}_t \right\rangle \right\}\right) \\
    &  \le O\left(2 + \inf_{\sum_i |y^{(i)}|\le \|u\| \sqrt{\frac{\log T}{\log(KT)}}} \left\{ 2\|u\|\log(T)+2\|u\|\sqrt{\sum_{t=1}^T \|c_t\|^2\log T} -\sum_{t=1}^T \left\langle c_t, \sum_{i=1}^Ky^{(i)}h^{(i)}_t \right\rangle \right\}\right).
\end{align*}
Let $w$ be an arbitrary element of $\Delta_K$. We set $y^{(i)} = \|u\|\frac{w^{(i})}{\sqrt{\alpha |B^{H(w)}_{\alpha}|+\frac{\log(KT)}{\log T}}}$. Notice that this implies $\sum |y^{(i)}|\le \|u\| \sqrt{\frac{\log T}{\log(KT)}}$. Also, we have
\begin{align*}
    -\sum_{t=1}^T \langle c_t, H(w)_t\rangle&\le -\sum_{t=1}^T \alpha \|c_t\|^2 +2|B_{\alpha}^{H(w)}|, \quad\mbox{ and } \\
    -\sum_{t=1}^T \left\langle c_t, \sum_{i=1}^Ky^{(i)}h^{(i)}_t \right\rangle&\le -\frac{\|u\|}{\sqrt{\alpha |B_{\alpha}^{H(w)}|+\frac{\log(KT)}{\log T}}}\sum_{t=1}^T \alpha \|c_t\|^2 +2\|u\| \sqrt{\frac{|B_{\alpha}^{H(w)}|}{\alpha}}.
\end{align*}
Thus the regret bound for $\cA$ becomes 
\begin{align*}
    \regret_{\cA}(u, \vec{c} \mid H)&\le O\left(2+w\|u\|\log(T) +2\|u\|\sqrt{\frac{|B_{\alpha}^{H(w)}|}{\alpha}}\right.\\
    &\kern3em\left.+ 2\|u\|\sqrt{\sum_{t=1}^T \|c_t\|^2\log T} -\frac{\|u\|}{\sqrt{\alpha |B_{\alpha}^{H(w)}|+\frac{\log(KT)}{\log T}}}\sum_{t=1}^T \alpha \|c_t\|^2 \right)\\
    &\le O\left(2 + \frac{\|u\| (\log T)\sqrt{\alpha |B_{\alpha}^{H(w)}|+\frac{\log(KT)}{\log T}}}{\alpha} +2\|u\|\sqrt{\frac{|B_{\alpha}^{H(w)}|}{\alpha}}\right)\\
    &=O\left(\frac{\|u\|\sqrt{(\log T)\log(KT)}}{\alpha} + \|u\|(\log T) \sqrt{\frac{|B_{\alpha}^{H(w)}|}{\alpha}}\right).
\end{align*}
Since $w$ was chosen arbitrarily in $\Delta_K$, the bound holds for all $w\in\Delta_K$ and so we are done.
\end{proof}
\fi
%





\section{Conclusions}

In this paper we obtained algorithms for online linear optimization in the presence of many hints that can be imperfect.  Besides generalizing previous results on online optimization with hints, our contributions include a simple algorithm for combining arbitrary learners that seems to have broader applications. Interesting future research directions include tightening the dependence on $\alpha$ in various cases and exploring the possibility of improved bounds for specific online  optimization problems. 

\bibliographystyle{plain}
\bibliography{mhints}

\clearpage
\appendix
\section{Single hint setting}\label{sec:betteralphasinglehint}

In this section, we modify the construction of \cite{bhaskara2020online} in the single hint setting to take into account knowledge of the parameter $\alpha$. Our goal is to prove Theorem \ref{thm:betteralphasinglehint}. The algorithm is nearly identical to that of \cite{bhaskara2020online} and most of the analysis is the same. We refer the reader to the original reference for complete details. 

\begin{algorithm}[ht]
   \caption{$\singlehint_\alpha$}
   \label{alg:singlehint}
   \begin{algorithmic}
   \REQUIRE Parameter $\alpha$
    \STATE Define $\lambda_0 =1$ and $r_0=1$
    \STATE Set procedure $\cA$ to be Algorithm 2 in \cite{bhaskara2020online}.
    \FOR{$t=1, \dots, T$}
    	\STATE Get hint $h_t$
    	\STATE Get  $\xbar_t$ from procedure $\cA$, and set 
    	
    	\mymath{ 
    	x_t \gets \xbar_t + \frac{(\norm{\xbar_t}^2-1)}{2r_t} h_t
    	}
	\STATE Play $x_t$ and receive cost $c_t$
	\STATE Set $r_{t+1} \gets \sqrt{r_t^2 + \frac{\alpha\max(0, -\iprod{c_t, h_t})}{\log(T)}}$
	\STATE Define $\sigma_t = \frac{|\iprod{c_t, h_t}|}{r_t}$
	\STATE Define $\lambda_t$ as the solution to:
	
	\mymath{
	\lambda_t = \frac{\norm{c_t}^2}{\sum_{\tau=1}^t\sigma_\tau +  \lambda_\tau}
	}
	
	\STATE Define the loss $\ell_t(x)= \iprod{ c_t, x} + \frac{|\iprod{c_t, h_t}|}{2r_t} (\norm{x}^2 -1)$.
	Send the loss function $\ell_t$ to $\cA$
	\ENDFOR
\end{algorithmic}
\end{algorithm}

The only difference between our algorithm $\singlehint_\alpha$ and Algorithm 1 of \cite{bhaskara2020online} is the definition of $r_t$: when we set $r_{t+1} = \sqrt{r_t^2 + \frac{\max(0, -\iprod{c_t, h_t})\alpha}{\log(T)}}$, \cite{bhaskara2020online} instead sets $r_{t+1} = \sqrt{r_t^2 +\max(0,-\iprod{c_t, h_t})}$.
We can now prove Theorem \ref{thm:betteralphasinglehint}, which we restate below for reference:

\thmbetteralphasinglehint*
\begin{proof}
Following \cite{bhaskara2020online}, we observe that since $\cA$ always returns $\xbar_t\in \cB$, $x_t\in \cB$. Further,
\begin{align*}
    \langle c_t, x_t - u\rangle \le \ell_t(x_t) - \ell_t(u) + \frac{\max(0, -\langle c_t, h_t\rangle)}{r_t},
\end{align*}
and $\ell_t$ is $\sigma_t$-strongly convex.

Next, by \cite{bhaskara2020online} Lemma 3.4, we have
\begin{align*}
    \regret_{\singlehint_\alpha}(\cB, \vec{c} \mid \{\vec{h}\})\le \sum_{t=1}^T \frac{\max(0, - \langle c_t, h_t\rangle)}{r_t}+\sum_{t=1}^T \ell_t(\bar x_t) -\ell_t(u).
\end{align*}

We can bound the first sum as:
\begin{align*}
    \sum_{t=1}^T \frac{\max(0, - \langle c_t, h_t\rangle)}{r_t}&\le \frac{\log T}{\alpha}\sum_{t=1}^T \frac{\alpha\max(0, - \langle c_t, h_t\rangle)/\log T}{r_t}\\
    &\le \frac{2\log T}{\alpha}\sqrt{\sum_{t=1}^T \frac{\alpha\max(0, - \langle c_t, h_t\rangle)}{\log T}}\\
    &\le \sqrt{2\frac{\sum_{t=1}^T (\log T)\max(0, - \langle c_t, h_t\rangle)}{\alpha}}.
\end{align*}
For the second sum, we appeal to Lemma 3.6 of \cite{bhaskara2020online}, which yields:
\begin{align*}
    \sum_{t=1}^T \ell_t(\bar x_t) -\ell_t(u)&\le \frac{1}{2} + 4\left(\sqrt{\sum_{t\in B^{\vec{h}}_{\alpha}}\|c_t\|^2} + \frac{r_T(\log T)}{\alpha}\right)\\
    &\le \frac{1}{2} + 4\left(\sqrt{\sum_{t\in B^{\vec{h}}_{\alpha}}\|c_t\|^2} + \frac{\sqrt{(\log^2 T) + (\log T)\alpha\sum_{t=1}^T \max(0,-\langle c_t, h_t\rangle)}}{\alpha}\right)\\
    &\le \frac{1}{2} + 4\left(\sqrt{\sum_{t\in B^{\vec{h}}_{\alpha}}\|c_t\|^2} + \frac{\log T}{\alpha}+\sqrt{\frac{(\log T)\sum_{t=1}^T \max(0,-\langle c_t, h_t\rangle)}{\alpha}}\right).
\end{align*}
Combining these identities now yields the desired theorem.
\end{proof}


\section{Full proofs: Constrained setting}\label{app:constrained}



Before proving Proposition~\ref{thm:selfbounding}, we apply the analysis of adaptive follow-the-regularized-leader (FTRL) as in~\cite{mcmahan2017survey} to obtain:

\begin{restatable}{proposition}{wregret}\label{thm:wregret}
For any $w_\star \in \Delta_K$, we have:
\begin{align*}
    \sum_{t=1}^T (\ell_t(w_t) -\ell_t(w_\star)) \le 2\sqrt{(\log^2 K)+(\log K)\sum_{t=1}^T \|g_t\|_\infty^2}.
\end{align*}
\end{restatable}
\begin{proof}
To begin, recall that the entropic regularizer $\psi(w) = \log(K)+\sum_{i=1}^K w^{(i)} (\log w^{(i)})$ is 1-strongly-convex with respect to the 1-norm over $\Delta_K$, has minimum value 0 and maximum value $\log K$.

Then, standard bounds for FTRL (e.g., \cite[Theorem 1]{mcmahan2017survey}) tell us that:
\begin{align*}
     \sum_{t=1}^T \ell_t(w_t) -\ell_t(w_\star)&\le  \sqrt{\frac{(\log K)+\sum_{t=1}^T \|g_t\|_\infty^2}{\log K}}\psi(w_\star) +\sum_{t=1}^T \frac{\|g_t\|_\infty^2\sqrt{\log K}}{2\sqrt{(\log K)+\sum_{\tau=1}^{t-1}\|g_\tau\|_\infty^2}} \\
     &\le  \sqrt{\frac{(\log K)+\sum_{t=1}^T \|g_t\|_\infty^2}{\log K}}\psi(w_\star) +\sum_{t=1}^T \frac{\|g_t\|_\infty^2\sqrt{\log K}}{2\sqrt{\sum_{\tau=1}^{t}\|g_\tau\|_\infty^2}}\\
     &\le \sqrt{\frac{(\log K)+\sum_{t=1}^T \|g_t\|_\infty^2}{\log K}}\psi(w_\star)+ \sqrt{(\log K)\sum_{t=1}^{T}\|g_t\|_\infty^2}\\
     &\le 2\sqrt{(\log^2 K) + (\log K) \sum_{t=1}^{T}\|g_t\|_\infty^2}.
\end{align*}
\end{proof}

Now with Proposition \ref{thm:wregret} in hand, we can restate and prove:
\selfbounding*
\begin{proof}
From Proposition \ref{thm:smoothprop}, we have
\begin{align*}
    \sum_{t=1}^T \|g_t\|_\infty^2\le \sum_{t=1}^T \frac{4}{\alpha}\ell_t(w_t).
\end{align*}
Combining this with the regret bound of Proposition \ref{thm:wregret} yields:
\begin{align*}
    \sum_{t=1}^T \ell_t(w_t) - \ell_t(w_\star)&\le 2\sqrt{(\log^2 K)+\frac{4\log K}{\alpha}\sum_{t=1}^T \ell_t(w_t)}.
\end{align*}
If we set $R=\sum_{t=1}^T \ell_t(w_t)-\ell_t(w_\star)$, we can rewrite the above as:
\begin{align*}
    R\le 2\sqrt{(\log^2 K) + \frac{4\log K}{\alpha}R + \frac{4\log K}{\alpha}\sum_{t=1}^T \ell_t(w_\star)}.
\end{align*}
Now we use $\sqrt{a+b}\le \sqrt{a}+ \sqrt{b}$ and solve for $R$:
\begin{align*}
    R&\le \frac{16\log K}{\alpha} + \sqrt{4 \log^2 K + \frac{16\log K}{\alpha}\sum_{t=1}^T \ell_t(w_\star)}\\
    &\le \frac{18\log K}{\alpha} + \sqrt{\frac{16\log K}{\alpha}\sum_{t=1}^T\ell_t(w_\star)}\\
    \implies
    \sum_{t=1}^T \ell_t(w_t)&\le \sum_{t=1}^T \ell_t(w_\star) + \frac{18\log K}{\alpha} + \sqrt{\frac{16\log K}{\alpha}\sum_{t=1}^T\ell_t(w_\star)}.
\end{align*}
Next, observe that $\sqrt{aX}\le X + \frac{a}{4}$, so that
\begin{align*}
    \sum_{t=1}^T \ell_t(w_t)&\le 2\sum_{t=1}^T \ell_t(w_\star) + \frac{22\log K}{\alpha}.
\end{align*}
as desired.
\end{proof}

\section{Proofs from Section~\ref{sec:combiner}}
\label{app:combiner}

\Kcombiner*
\begin{proof}
We can divide the operation of Algorithm \ref{alg:combiner} into phases in which $\gamma$ is constant. Each phase may be further subdivided into sub-phases in which $i$ is constant. First, let us bound the regret in a single phase with fixed $\gamma$. Suppose this phase has $N \leq K$ sub-phases\footnote{All phases except maybe the last phase have exactly $K$ sub-phases.}. Let $t_1,\dots,t_N$ be the time indices at which each sub-phase begins, and let $t_{N+1}-1$ be the last time index belonging to this phase. Notice that for all $i\le N$, we must have $r^{i,\gamma}_{t_{i+1}-t_i-1}\le \gamma$ since the $i$th sub-phase lasts for $t_{i+1}-t_i$ iterations. Then since $\sup_{x,y}\langle c_{t_{i+1}-1},x-y\rangle\le 1$ for all $i$ and $x,y\in X$, we have $r^{i,\gamma}_{t_{i+1}-t_i}\le r^{i,\gamma}_{t_{i+1}-t_i-1} +1\le \gamma + 1$.
Now we can write the regret incurred over this phase as:
\begin{align*}
    \sup_{u\in X} \sum_{t=t_1}^{t_{N+1}-1} \langle c_t, x_t - u\rangle&\le \sum_{i=1}^N \sup_{u\in X} \sum_{t=t_i}^{t_{i+1}-1} \langle c_t, x_t-u\rangle
    &\le \sum_{i=1}^N  r^{i,\gamma}_{t_{i+1} - t_i}
    \le N(\gamma +1)
    \le K\gamma + K.
\end{align*}

Let $P$ denote the total number of phases. We now show that $P\le 2 +\max(-1, \log_2\left(\min_i  \Rci{i}{1}{T}{\vec{c}}\right))$. Suppose otherwise. Let $j=\argmin_i  \Rci{i}{1}{T}{\vec{c}}$ be the algorithm with the least total regret. Let us consider the $(P-1)$th phase. In this phase, $\gamma=2^{P-2}$. Since $P> 2+ \log_2\left(\min_i  \Rci{i}{1}{T}{\vec{c}}\right)$, we must have $\min_i  \Rci{i}{1}{T}{\vec{c}} < \gamma$. Consider the $j$th sub-phase in this phase. Since $\gamma$ will eventually increase, this sub-phase must eventually end. Therefore there must be some $t$ and $\tau$ such that $t+\tau<T$ and
\begin{align*}
    \sup_{u\in X} \sum_{\tau'=1}^{\tau} \langle c_{t+\tau'}, w_{\tau'} - u\rangle > \gamma,
\end{align*}
where $w_{\tau'}$ is the output of $A_j$ after seeing input $c_t,\dots,c_{t+\tau'-1}$. By the increasing property of $R_j$, we also have:
\begin{align*}
    \sup_{u\in X} \sum_{\tau'=1}^{\tau} \langle c_{t+\tau'}, w_{\tau'} - u\rangle \le \Rci{j}{t}{t+\tau}{\vec{c}} \le \Rci{j}{1}{T}{\vec{c}} <\gamma.
\end{align*}
which is a contradiction. Therefore $P\le  2 +\max(-1, \log_2\left(\min_i \Rci{i}{1}{T}{\vec{c}}\right))$.

Now we are in a position to calculate the total regret. Let $1=T_1,\dots,T_P$ be the start times of the $P$ phases, and let $T_{P+1}-1=T$ for notational convenience. Then we have:
\begin{align*}
    \sup_{u\in X} \sum_{t=1}^T\langle c_t, x_t - u\rangle &\le \sum_{e=1}^P \sup_{u\in X} \sum_{t=T_e}^{T_{e+1}-1}\langle c_t, x_t - u\rangle. \\
    \intertext{Now since the regret in an phase is at most $K\gamma +K$, and $\gamma$ doubles every phase,}
    &\le \sum_{e=1}^P K2^{e-1} + K
    \le KP + K2^{P}\\
    &\le K2^{P+1}\\
    &\le K\left(4+4\min_i \Rci{i}{1}{T}{\vec{c}}\right), 
\end{align*}
where the second-to-last inequality follows from $x\le 2^{x}$ for $x\ge 1$, and the last inequality is from case analysis.
\end{proof}

\begin{algorithm}[H]
\caption{Randomized combiner.}\label{alg:logkcombiner}
\begin{algorithmic}
  \STATE{\bfseries Input:} Online algorithms $\cA_1,\dots,\cA_K$
  \STATE Reset $\cA_1$
  \STATE Set $\gamma \leftarrow 1$, $\tau \leftarrow 1$
  \STATE Initialize the candidate indices $C \leftarrow [K]$
  \STATE Choose index $i$ uniformly at random from $C$
  \FOR{$t=1, \dots, T$}
  \FOR{$j\in C$}
  \STATE Get $y^j_\tau$, the $\tau$th output of $\cA_j$ 
  \ENDFOR
  \STATE Respond $x_t \leftarrow y^i_\tau$
  \STATE Get cost $c_t$, define $g_\tau \leftarrow c_t$
  \FOR{$j\in C$}
  \STATE Send $g_\tau$ to $\cA_j$ as $\tau$th cost
  \STATE Set $r^{j,\gamma}_\tau \leftarrow \sup_{u\in \cB}\sum_{\tau'=1}^\tau \langle g_{\tau'}, y^j_{\tau'}- u\rangle$
  \IF{$r^{j,\gamma}_\tau >\gamma$}
  \STATE Set $C \leftarrow C \setminus \{j\}$
  \ENDIF
  \ENDFOR
  \IF{$i\notin C$}
  \IF{$C=\emptyset$}
  \STATE Set $C \leftarrow [K]$
  \STATE Set $\gamma \leftarrow 2\gamma$
  \ENDIF
  \STATE Set $\tau \leftarrow 1$
  \STATE Reset $\cA_j$ for all $j\in C$
  \STATE Select index $i$ uniformly at random from $C$
  \ENDIF
  \STATE Set $\tau \leftarrow \tau + 1$
  \ENDFOR
\end{algorithmic}
\end{algorithm}

\logkcombiner*

\begin{proof}
We divide the operation of Algorithm \ref{alg:logkcombiner} into phases in which $\gamma$ is constant. Each phase is further subdivided into sub-phases in which $i$ is constant. First, let us fix an phase $e$ with a fixed value of $\gamma$ and bound the expected regret incurred in this phase. Let $N$ denote the number of sub-phases in this phase. Just as in the proof of Theorem \ref{thm:Kcombiner}, we can show that the total regret incurred in this phase is at most $N(\gamma + 1)$. However, while there are exactly $K$ sub-phases in any phase of Algorithm \ref{alg:combiner} (except perhaps the last one), the number of sub-phases in any phase of Algorithm \ref{alg:logkcombiner} is a random variable. 

We now bound $\E[N]$, the expected number of sub-phases in any phase.
For the fixed phase $e$, for any time index $t$, let $F(i,t)$ be the smallest index $\tau \geq t$ such that $\sup_{u\in X}\sum_{\tau'=t}^\tau \langle c_{\tau'},w^i(t,\tau')-u\rangle >\gamma$, where we define $w^i(t,\tau')$ to be the output of $A_i$ after seeing input $c_t,\dots,c_{\tau'-1}$ and $w^i(t,t)$ to be the initial output of $A_i$. We set $F(i,t)=T$ if no such index $\tau\le T$ exists. Intuitively, $F(i,t)$ denotes the index $\tau \geq t$ when the regret experienced by algorithm $A_i$ that is initialized at time $t$ first exceeds $\gamma$.

Let $C(S,t)$ be the expected number of sub-phases (counting the current one) left in the phase if a sub-phase starts at time $t$ with the specified set of active indices $S$. We define $C(S,T+1)=C(\emptyset, t)=0$ for all $S$ and $t$ for notational convenience. Note that $C(S,T)=1$ for all $S$. Further, by definition, we have $\E[N] = C(\{1, 2, \ldots, K\}, t)$ for some $t$ (corresponding to the start of the phase). We claim that $C$ satisfies:
\begin{align*}
    C(S,t) = 1+ \frac{1}{|S|}\sum_{i\in S} C(S \setminus \{j\in S ~\mid~ F(j,t)\le F(i,t)\}, F(i,t)+1).
\end{align*}
To see this, observe that each index $i\in S$ is equally likely to be selected for the fixed $i$ throughout the sub-phase starting at time $t$. By definition of $F$, the sub-phase will end at time $F(i,t)$ if the selected index is $i$. Further, at the end of the sub-phase, $S$ will be $S \setminus \{j\in S ~\mid~ F(j,t)\le F(i,t)\}$. Therefore, conditioned on selecting index $i$ for this sub-phase, the expected number of sub-phases is $1+C(S \setminus \{j\in S ~\mid~ F(j,t)\le F(i,t)\}, F(i,t)+1)$. Since each index is selected with probability $1/|S|$, the stated identity follows.
Now we apply Lemma~\ref{thm:logrecursion} to conclude that $C(\{1,\dots,K\},t)\le \log_2(K+1)$ for all $t$, which implies $\E[N] \leq \log_2(K+1)$.

Finally, let $P$ denote the total number of phases. We can show that $P \leq 2 + \max(-1, \log_2(\min_i \Rci{i}{1}{T}{\vec{c}}))$. The proof of this claim is identical to that in Theorem~\ref{thm:Kcombiner} and is omitted for brevity.
Let $N_p$ and $\gamma_p = 2^{p-1}$ denote the number of sub-phases in phase $p$ and the corresponding value for $\gamma$ respectively. We can then conclude the total expected regret experienced by Algorithm~\ref{alg:logkcombiner} is
\begin{align*}
    \E\left[\sup_{u\in X}\sum_{t=1}^T \langle c_t, x_t-u\rangle\right] &\leq \sum_{p=1}^P \E[N_p](\gamma_p + 1) \leq (2^P + P) \cdot \log_2(K+1)\\
    &\leq \log_2(K+1)\left(4 + 4\min_i \Rci{i}{1}{T}{\vec{c}}\right).
\end{align*}

To prove the second bound for an adaptive adversary, we simply observe that in the worst-case, we cannot have more than $K$ sub-phases in any phase. The rest of the argument is identical.
\end{proof}

In order to prove Theorem \ref{thm:logkcombiner}, we need the following technical Lemma: 

\begin{restatable}{lemma}{logrecursion}\label{thm:logrecursion}
Let $F: [K] \times [T] \rightarrow [T]$ be such that $F(i,t) \geq t$ for all $i \in [K],t \in [T]$ and $C: 2^{[K]} \times [T] \rightarrow \R$ be a function that satisfies $C(\emptyset, t) = 0$ for all $t$, $C(S,T)=1$ for all $S$, $C(S,T+1)=0$ for all $S$, and $C$ satisfies the recursion:
\begin{align*}
    C(S,t) = 1+\frac{1}{|S|}\sum_{i\in S}C(S \setminus \{j\in S ~\mid~ F(j,t)\le F(i,t)\},F(i,t)+1).
\end{align*}
Then $C(\{1,\dots,K\}, t)\le \log_2(K+1)$ for all $t$.
\end{restatable}
\begin{proof}
We define the auxiliary function $Z(N)=\sup_{t, |S|\le N}C(S,t)$. Observe $Z(0)=0$, $Z(1)=1$, and $Z(N)$ is non-decreasing with $N$. Now suppose for purposes of induction that $Z(n)\le \log_2(n+1)$ for $n<N$. Then we have
\begin{align*}
    Z(N)&\le 1+\sup_{N'\le N}\frac{1}{N'}\sup_{t,|S|=N'}\sum_{i\in S}C(S-\{j\in S ~\mid~ F(j,t)\le F(i,t)\},F(i,t)+1)\\
    &\le 1+ \sup_{N'\le N}\frac{1}{N'}\sup_{t,|S|=N'}\sum_{i\in S}Z(N'-|\{j\in S ~\mid~ F(j,t)\le F(i,t)\}|). \\
\intertext{Now since $Z(n)$ is non-decreasing in $n$, this is bounded by:}
    &\le 1+ \sup_{N'\le N}\frac{1}{N'}\sum_{i=1}^{N'}Z(N'-i)\\
    &\le 1+ \sup_{N'\le N}\frac{1}{N'}\sum_{i=1}^{N'}\log_2(N'-i+1). \\
    \intertext{Now we apply Jensen inequality to the concave function $\log_2(n)$:}
    &\le 1+ \sup_{N'\le N}\log_2\left(\frac{1}{N'}\sum_{i=1}^{N'} N'-i+1\right)\\
    &\le 1+\sup_{N'\le N} \log_2((N'+1)/2)\\
    &=\log_2(N+1).
\end{align*}
To conclude, note that clearly $C(\{1,\dots,K\},t)\le Z(K)$ for all $t$.
\end{proof}

\ifarxiv
\else

\fi

\end{document}